%% file: main.tex
\documentclass[journal]{IEEEtran}
\usepackage{amsmath,amsfonts,amssymb,amsthm}
\usepackage{algorithmic}
\usepackage{algorithm}
\usepackage{array}
\usepackage[caption=false,font=normalsize,labelfont=sf,textfont=sf]{subfig}
\usepackage{textcomp}
\usepackage{stfloats}
\usepackage{url}
\usepackage{verbatim}
\usepackage{graphicx}
\usepackage{cite}
\usepackage[hidelinks]{hyperref}
\usepackage{cleveref}
\usepackage{etoolbox}
\makeatletter
\patchcmd{\@makecaption}
  {\scshape}{}{}{}
\patchcmd{\@makecaption}
  {\\}{.\ }{}{}
\patchcmd{\@makecaption}
  {\centering}{}{}{}
\makeatother

\usepackage{xcolor}
\newcommand{\rev}[1]{#1}

\input{shortcuts}

\begin{document}

\title{Pivotal Auto-Encoder via Self-Normalizing ReLU}

\author{Nelson Goldenstein, Jeremias Sulam, Yaniv Romano
\thanks{N. Goldenstein is with the Department of Electrical and Computer Engineering at the Technion, Israel.}
\thanks{J. Sulam is with the Biomedical Engineering Department and the Mathematical Institute for Data Science of Johns Hopkins University, USA.}
\thanks{Y. Romano is with the Department of Electrical and Computer Engineering and the Department of Computer Science at the Technion, Israel.}}

\maketitle

\begin{abstract}
Sparse auto-encoders are useful for extracting low-dimensional representations from high-dimensional data. However, their performance degrades sharply when the input noise at test time differs from the noise employed during training. This limitation hinders the applicability of auto-encoders in real-world scenarios where the level of noise in the input is unpredictable. In this paper, we formalize single hidden layer sparse auto-encoders as a transform learning problem. Leveraging the transform modeling interpretation, we propose an optimization problem that leads to a predictive model invariant to the noise level at test time. In other words, the same pre-trained model is able to generalize to different noise levels. The proposed optimization algorithm, derived from the square root lasso, is translated into a new, computationally efficient auto-encoding architecture. After proving that our new method is invariant to the noise level, we evaluate our approach by training networks using the proposed architecture for denoising tasks. Our experimental results demonstrate that the trained models yield a significant improvement in stability against varying types of noise compared to commonly used architectures.
\end{abstract}

\begin{IEEEkeywords}
Sparse coding, transform learning, sparse auto-encoders, square root lasso.
\end{IEEEkeywords}

\input{1_intro}
\input{2_theory}
\input{3_algo}
\input{4_experiments}
\input{5_conclusion}

\section*{Acknowledgments}
N.G. and Y.R. were supported by the Israel Science Foundation (grant No. 729/21). Y.R. also thanks the Career Advancement Fellowship, Technion, for providing research support. J.S. is supported by NSF grant CCF 2007649.

\appendices
\input{appendix}

\rev{
\bibliographystyle{IEEEtran}
\bibliography{bib}
}

\end{document}

%% file: shortcuts.tex
\theoremstyle{definition}
\newtheorem{theorem}{Theorem}

\newtheorem{lemma}{Lemma}
\newtheorem{assumption}{Assumption}
\newtheorem{definition}{Definition}
\newtheorem{example}{Example}
\newtheorem{corollary}{Corollary}
\theoremstyle{remark}
\newtheorem{remark}{Remark}

\newcommand{\cA}{\mathcal{A}}

\newcommand{\cC}{\mathcal{C}}

\newcommand{\cN}{\mathcal{N}}

\newcommand{\cS}{\mathcal{S}}

\newcommand{\bP}{\mathbb{P}}

\newcommand{\bR}{\mathbb{R}}

\renewcommand{\geq}{\geqslant}
\renewcommand{\leq}{\leqslant}
\def\hat{\widehat}

\DeclareMathOperator*{\argmin}{\text{arg\,min}}

\newcommand{\eqdef}{\triangleq}
\newcommand{\iid}{{i.i.d.~}}
\newcommand{\norm}[1]{\lVert #1 \rVert}
\def\ReLU{{\text{ReLU}}}

%% file: 1_intro.tex
\section{Introduction} \label{sec:Introduction}
\IEEEPARstart{A}{}
sparse auto-encoder is a type of artificial neural network that learns efficient data encodings through unsupervised learning \cite{Goodfellow2016DL}. The purpose of an auto-encoder is to capture the most important elements of the input to learn a lower dimensional representation for higher dimensional data, such as images \cite{2010denoisingae}. It is commonly used for dimensionality reduction or feature extraction. The sparse auto-encoder architecture consists of two modules: an encoder and a decoder. The encoder compresses the input data into an encoded representation in a different domain, which is forced to be sparse. It then processes and filters the encoded data representations so that only the most important information is allowed through, preventing the model from memorizing the inputs and overfitting. The final module of the network, the decoder, decompresses the extracted sparse representations and reconstructs the data from its encoded state back to its original domain.

Interestingly, the observation that natural data can often be accurately approximated by sparse signals has been a prominent framework over the last twenty years \cite{Romano2015Denoising, Yang2010Image, Papyan2018Theory}. 
\rev{Specifically, the transform model \cite{SparseTransforms}---a generalized analysis sparse representation model---assumes that a signal $x\in\bR^n$ has a sparse representation $z^*\in\bR^d$ over a particular transformation $W \in \bR^{d \times n}$ to another domain, i.e.,
\begin{equation} \label{eq:transform}
    W x = z^*, \text{ where } \|z^*\|_0 \ll d,
\end{equation}
where $\|\cdot\|_0$ counts the number of nonzero elements of a vector.
This representation is typically a higher dimensional signal, i.e., $d \geq n$, and this is the setting that we assume in this work.
When $n = d$ and $W$ is of full rank, the transformation forms a basis, whereas when $d > n$, the transform is considered to be overcomplete. In situations where we observe a noisy version $y$ of the clean signal $x$, corrupted by additive noise, the equation becomes
\begin{equation*}
    W y = z^* + e,
\end{equation*}
where $e$ denotes the error or residual in the transform domain.
}

In this context, the task of finding the sparse representation of a signal given $W$ is called sparse coding. The first module of a sparse auto-encoder---the encoder---can be formally written as a transform sparse coding problem:
\begin{equation} \label{eq:encoder}
    \hat z = \argmin_z \frac{1}{2} \|z - W y\|^2_2 + \lambda \|z\|_1,
\end{equation}
where $\hat z$ is the estimated latent space representation, $W$ is a known transformation, and $\lambda \in \bR_+$ is a hyperparameter. The optimization problem in \eqref{eq:encoder} minimizes the transform residual $e$ with a sparse prior for $z$ to estimate $z^*$. \rev{The parameter $\lambda$ governs the trade-off between sparsity and residual error. Observe that the optimal selection of $\lambda$ is dependent on $e$ since  $z^* - W y = e$; hence, the value of $\lambda$ ought to be calibrated to increase with the magnitude of $e$.}

The closed form solution to the encoding problem \eqref{eq:encoder} is
\begin{equation*}
    \hat z = \text{prox}_{\lambda \|z\|_1} \big( W y \big) = S_{\lambda} \big(W y \big),
\end{equation*}
where $\text{prox}_f(v) = \argmin_x (f(x) + \tfrac{1}{2} \norm{x-v}^2)$ is the proximal operator and $S_\lambda$ is the soft-thresholding operator $S_\lambda(x)=\text{sign}(x)\text{max}(|x|-\lambda, 0)$. If we add the assumption of non-negativity of $z^*$, the solution can be rewritten as
\begin{equation*} \label{eq:relu}
    \hat z = \ReLU(W y - \lambda),
\end{equation*}
where $\ReLU(x)=\text{max}(x,0)$. From this expression we understand the influence of $\lambda$ on the solution and its role in filtering perturbations. Higher values of $\lambda$ lead to lower and sparser solutions. \rev{Therefore, as mentioned earlier, the optimal value of $\lambda$ is a function of the noise; the stronger the noise, the larger $\lambda$ must be.}

The decoder can also be written as an inverse operation following the same model. Formally, given $W$ and $\hat z$, we can obtain a least squares approximation to the true signal $x$ by minimizing $\|W x - \hat z\|^2_2$ with respect to $x$. Thus, the recovered signal is
\begin{equation*}
    \hat x = W^+ \hat z, 
\end{equation*}
where $W^+$ is the pseudoinverse of $W$. In particular, if $W$ has full column rank, $W^+ = (W^{\mathsf{T}} W)^{-1} W^{\mathsf{T}}$.

The connection between the transform model \eqref{eq:transform} and sparse auto-encoders is clear. The linear transformation $W$ represents the weights of any combination of linear layers, including convolutional operations, and $\lambda$ is the bias parameter; and both are trainable parameters of the network. This connection has been successfully applied to numerous computer vision and image processing tasks \cite{ravishankar2015mri}. When adapted to deep learning, it has demonstrated state-of-the-art performance in various machine learning applications, such as image classification~\cite{maggu2018CTL}, online video denoising~\cite{wen2018vidosat}, semantic segmentation of images~\cite{wen2017Frist}, super-resolution~\cite{gigie2021joint}, clustering~\cite{maggu2020DCTL}, and others.

\rev{The main problem of the presented encoder algorithm~\eqref{eq:encoder} is the bias parameter $\lambda$. The optimal selection of $\lambda$ is influenced by the noise, which means its ideal value varies with differing noise levels. Indeed, the optimality conditions for the correct estimation of the sparse representation $z^*$ are not guaranteed at different noise intensities, even for a known $W$. In other words, the model's performance may deteriorate significantly if the noise level at test time is different from the one used during training~\cite{mohan2019robust}.} Therefore, the network must be re-trained, and a different bias must be learned \emph{for each noise level} \cite{lecouat2019Interpretable}. Moreover, in an environment where the noise is unknown, as in most practical cases, finding the best bias becomes infeasible: it is impossible to choose the correct bias without estimating the noise level.

\textbf{Our contribution.}
To overcome the dependence of the bias $\lambda$ on the noise level, we draw inspiration from the square root lasso problem, introduced by \cite{Belloni2011sqrt} and detailed in \Cref{subsec:synthesis}, and propose a modification of the transform sparse coding algorithm for the encoder module
\begin{equation} \label{eq:our}
    \hat z = \argmin_z \|z - W y\|_2 + \lambda \|z\|_1.
\end{equation}
Notice that the residual term is no longer quadratic. The main advantage of \eqref{eq:our} is that the hyperparameter $\lambda$ is now pivotal to the noise energy. In other words, \textbf{the optimal choice of $\lambda$ is independent of the noise level}, which is difficult to estimate reliably because it is an ill-posed problem \cite{Naka2019NoiseEstimation}. We prove in \Cref{sec:Theory} that this property holds in the presence of both bounded noise and additive Gaussian noise. This stands in sharp contrast to vanilla sparse auto-encoders, where one needs to know the true standard deviation of the noise to fit the bias of the original transform sparse coding problem~\eqref{eq:encoder}.

\rev{
Furthermore, we propose an efficient and differentiable algorithm to solve the new pivotal sparse coding problem based on proximal gradient descent, as described in \Cref{sec:Algo}. Our algorithm is differentiable in the sense that it is compatible with gradient-based optimization techniques, enabling the minimization of a cost function through methods such as automatic differentiation and backpropagation. This leads to the development of a novel non-linear function called Self Normalizing ReLU (NeLU), which easily integrates into common neural network architectures. In \Cref{sec:Exp}, we conduct numerical experiments using both synthetic and real data to illustrate how our approach is significantly more resilient to various noise levels.
}

\section{Related work}
\subsection{Synthesis sparse modeling of signals} \label{subsec:synthesis}
\rev{
We begin by introducing a framework parallel to the analysis sparse representation model, presented in \Cref{sec:Introduction}, named the synthesis model. This model serves as the premise for introducing the square root lasso algorithm, which we will extend to the transform model.} The synthesis model assumes that a signal $x\in\bR^n$ can be represented as a linear combination of a few columns, called atoms, from a matrix $D\in\bR^{n \times d}$ named dictionary. In plain language, the signal corresponds to a multiplication of a dictionary by a sparse vector $z^*\in\bR^d$, i.e.,
\begin{equation*}
    x = D z^*.   
\end{equation*}

Various algorithms have been proposed to implement the sparse coding task of estimating $z^*$ given $x$ \cite{2006sparsecoding}. These include the matching pursuit~\cite{Mallat2013MP} and the basis pursuit~\cite{Chen2001BP}, also called lasso in the statistics literature~\cite{Tibshirani1996Lasso}
\begin{equation} \label{eq:lasso}
    \min_z \frac{1}{2n}\|y-D z\|^2_2 + \lambda_{L} \|z\|_1.
\end{equation}
\rev{As with \eqref{eq:encoder}, the primary drawback of lasso is that the optimal value of the parameter $\lambda_L$ is dependent on the noise level and, therefore, must be adjusted for each specific noise level.} For example, in a Gaussian noise environment, $\lambda_L$ proportional to $\sigma \sqrt{\log d/n}$ is minimax optimal for signal reconstruction, $\|\hat{x}-x\|_2$, in high dimensions~\cite{li2020fast}. Therefore, to achieve a good estimation of $z^*$ or prediction of $x$ for an unknown noise level, one must estimate $\sigma$.

\textbf{Square root lasso.}
A modified version of the lasso has been proposed to solve the dependence of $\lambda$ on noise power, called square root lasso~\cite{Belloni2011sqrt}
\begin{equation} \label{eq:sqrt_lasso}
    \min_z \frac{1}{\sqrt{n}}\|y- D z\|_2 + \lambda_{\sqrt{L}} \|z\|_1,
\end{equation}
which takes the square root of the error term of \eqref{eq:lasso}. Belloni et al.~\cite{Belloni2011sqrt} have proven that the square root lasso achieves minimax optimality of estimation and signal reconstruction error for a hyperparameter $\lambda_{\sqrt{L}}$, which is pivotal to the noise level. For instance, in the case of Gaussian noise, the minimax optimal $\lambda_{\sqrt{L}}$ is proportional to $\sqrt{\log d / n}$, which is independent of $\sigma$ and leads to a constant parameter for all Gaussian distributions. Moreover, numerous algorithms have been developed to efficiently solve the problem based on its convex property~\cite{li2020fast}.

Although the square root lasso is powerful and attractive, it has never been applied in the context of dictionary learning or adapted to form a novel neural network architecture. In this work, we draw an exciting connection between the square root lasso and transform modeling. The extension to synthesis model and dictionary learning is a direct consequence of the analysis of the transform learning since \eqref{eq:our} can be seen as a particular case of \eqref{eq:sqrt_lasso}, where $D=I$ and the input signal is $W y$. Further details are provided in Appendix \ref{App:synthesis}.

\subsection{Transform learning}
\rev{
As elaborated in \Cref{sec:Introduction}, the transform model and sparse auto-encoders are tightly connected. The forward pass in sparse encoders essentially acts as a sparse representation pursuit within the transform model. In this framework, the transformation is characterized by the weights within a set of layers, which may include convolutional ones. Thus, the transformation is also learned during the model training process. This method of deriving the transformation directly from the data is known as transform learning.
}

\rev{
At its core, transform learning~\cite{Bresler2015TL} employs a data-driven feature extractor to transform input data into a suitable representation. This approach can improve upon the limited ability of analytical transformation methods, such as wavelets, to handle data. As a result, transform learning often yields superior restoration performance compared to the analytical approaches~\cite{SparseTransforms}.
}

Transform learning is comparable to dictionary learning from an analysis perspective \cite{aberdam2019holistic}. In dictionary learning, a basis $D$ is trained to recover the data, $x = Dz^*$, from the representation $z^*$ \cite{2003dictionary}. In contrast, transform learning aims to learn a transformation $W$ to generate the representation $z^*$. Since the transformation $W$ is data-dependent and the output representation $z^*$ is unknown, jointly learning both is a challenging task.

In this work, we propose a new learning problem that combines the transform model with deep learning through the interpretable framework of transform learning. Our objective is to demonstrate that our method can exhibit the interesting properties of the classical problem established in \Cref{sec:Theory}. We extend our algorithm to the transform learning framework and demonstrate its effectiveness in enhancing the robustness of deep learning through experiments in \Cref{sec:Exp}.

\subsection{Blind denoising networks}
Blind denoising is the task of removing noise from an input signal when the noise magnitude is unknown at test time. This task is closely related to our objective. Several methods have been proposed to tackle this task, with the most common approach for blind denoising being to estimate the noise distribution (or simply the noise level) to identify and remove it from the signal~\cite{zhang2018FFD, Zhussip2019deep, chen2020blind}. However, this method lacks flexibility because it requires learning different weights for each noise level and solving the difficult problem of noise estimation.

In contrast to this approach, where all weights are learned from scratch for each noise level, some existing methods recognize that not all weights depend on the noise and only adjust the regularization parameters, such as the bias in the case of auto-encoders, for different noise levels \cite{lecouat2019Interpretable}. This method reduces the overall number of network parameters to be learned, but they still need to be adjusted for each noise level.

Another standard technique, which emerged alongside the advancement of deep learning in computer vision tasks, is to train one model across a wide range of expected noise levels~\cite{zhang2017Gaussian}. In this case, the denoising performance of such a model is generally inferior compared to a model trained for a specific noise distribution~\cite{2020OneSizeFitsAll}. Moreover, it has been shown that a model trained using this method tends to focus on the average noise level of the training range, rather than learning generalizable weights for all noise levels. To address this issue, Gnansambandam et al.~\cite{2020OneSizeFitsAll} proposed determining the optimal noise training sample distribution from a minimax risk optimization perspective. The approach proposed in \cite{2020OneSizeFitsAll} is orthogonal to ours, as it does not suggest modifying the network architecture but instead focuses solely on the training strategy.

In this work, we propose a new architecture for implementing \eqref{eq:our} that is inherently noise level independent. Our theoretical study, presented in \Cref{sec:Theory}, shows that the same model parameters achieve high-quality signal recovery across all noise levels when learned correctly. Indeed, our experimental results indicate that a single neural network can be trained and practically applied to handle all noise levels, without re-training or updating the bias term.

%% file: 2_theory.tex
\section{Theory} \label{sec:Theory}
We begin by formally restating the transform model from \Cref{sec:Introduction}. Specifically, we consider a sparse linear model in high dimensions for a noisy signal
\begin{equation*}
    y = x + \xi,
\end{equation*}
where $x \in \bR^n$ is the clean sparsifiable signal and $\xi \in \bR^n$ denotes the random additive noise. Thus, applying a given transformation $W$ to $y$ yields
\begin{equation*}
    W y = W(x + \xi) = W x + W \xi = z^* + e.
\end{equation*}
Here, the goal is to recover the clean sparse representation $z^* = W x$, while $e = W \xi$ is the error.

In the following subsections, we demonstrate that the desired properties of the square root lasso also hold for the sparse encoding problem \eqref{eq:our}, \rev{given a known transform $W$}. Specifically, the parameter $\lambda$ is pivotal to the noise level, and as a result, not only can the solution of \eqref{eq:our} be computed efficiently, but all parameters of the problem are also independent of the noise level. We prove that the recovery of the correct support, i.e., the group of nonzero elements $\{i \in [d] : |z^*_i| > 0 \}$, and the bound on the estimation error, $\|\hat z - z^*\|_2$, can be extended to our proposed new optimization problem under the presence of bounded noise.

\subsection{Support recovery} \label{subsec:support}
First, we prove the recovery of the correct support in the presence of bounded noise, a prevalent scenario in the robustness literature~\cite{Tempo1988Bounded, romano2020adversarial}. Subsequently, we extend these results to Gaussian noise within a probabilistic setting. To prove these results, the following assumptions are necessary.

\begin{assumption} \label{ass:bound_noise}
The noise $\norm{\xi}_2$ is bounded.
\end{assumption}

\noindent Consequently, $\norm{e}_2$ is also bounded since
$$
\norm{e}_2 \leq s_{\max} \norm{\xi}_2 \eqdef \epsilon ,
$$
where $s_{\max}$ is the largest singular value of $W$.

\begin{assumption} \label{ass:eta}
\rev{There exists a minimal positive number $\eta$ that satisfies the condition}
$$
\lambda \|z^*\|_1 \leq \eta \epsilon.
$$
\end{assumption}

\noindent The constant $\eta$ represents the ratio between the regularization term $\lambda \|z^*\|_1$ and the noise threshold $\epsilon$, which is the value of the reconstruction error term when $\hat z = z^*$. A smaller value of $\eta$ implies that the regularization term is proportionally smaller, leading to reduced shrinkage and potentially more accurate signal recovery.

\begin{theorem} \label{Thm:supp}
Let Assumption \ref{ass:bound_noise} be satisfied, let $\eta$ satisfy Assumption \ref{ass:eta}. Then, for
\begin{equation*}
    \lambda = \frac{\norm{e}_{\infty}}{\norm{e}_2},
\end{equation*}
and $\hat z$ be the solution to \eqref{eq:our}, we get that
\begin{equation*}
    \|{\hat z - z^*}\|_\infty \leq \lambda (2 + \eta) \epsilon.
\end{equation*}
Moreover, if
\begin{equation*}
    \min_{j \in [d]} |z_j^*| > 2 \lambda (2 + \eta) \epsilon,
\end{equation*}
then the estimated support
\begin{equation*} 
    \hat \cS = \{j \in [d] : |\hat z_j| > \lambda (2 + \eta) \epsilon \}
\end{equation*}
recovers the true sparsity pattern $\cS = \{j \in [d] : |z_j| > 0 \}$ correctly, i.e.,
\begin{equation*}
    \hat \cS = \cS.
\end{equation*}
\end{theorem}

\begin{proof} [Proof of \Cref{Thm:supp}]
Our derivation is inspired by the one in~\cite{massias2020support}. For the solution $\hat z$ of \eqref{eq:our}, we have:
\begin{align} \label{eq:bound_delta}
    \norm{\hat z - z^*}_\infty
    &= \norm{\hat e - e}_{\infty} \qquad \left(\hat e = W y - \hat z \right) \nonumber \\
    & \leq
    \norm{\hat e}_{\infty} + \norm{e}_{\infty} \nonumber \\
    & \leq
    \norm{\hat e}_{\infty} + \lambda \epsilon.
\end{align}
We bound $\norm{\hat e}_{\infty}$ using KKT sub-gradient optimality conditions,
\begin{align} \label{eq:kkt_sqrt}
    \norm{\hat e}_{\infty}
    &\leq \lambda \norm{\hat e}_2.
\end{align}
It now remains to bound $\norm{\hat e}_2$, which is done with \Cref{ass:eta}.
By minimality of the estimator,
\begin{align} \label{eq:link_eps}
    \norm{\hat e}_2 + \lambda \norm{\hat z}_{1}
    &\leq
    \norm{e}_2 + \lambda \norm{z^*}_{1} \nonumber \\
    \norm{\hat e}_2
    &\leq
    \norm{e}_2 + \lambda \norm{z^*}_{1} \nonumber \\
    &\leq
    \epsilon + \eta \epsilon. 
\end{align}
Combining equations \eqref{eq:bound_delta}, \eqref{eq:kkt_sqrt} and \eqref{eq:link_eps}, we get:
\begin{gather*}
    \norm{\hat z - z^*}_\infty \leq
    \norm{\hat e}_{\infty} + \lambda \epsilon \leq
    \lambda (\eta+1) \epsilon + \lambda \epsilon.
\end{gather*}
This proves the bound on $\norm{\hat z - z^*}_{\infty}$.
Finally, the correct support recovery follows directly from Theorem 4.1 in~\cite{lounici2009taking}.
\end{proof}

\begin{remark}
It is essential to highlight that $\lambda$ maintains the same value across different noise levels.
\rev{For example, let $\xi_1$ be a realization from a standard normal distribution $\cN(0, I)$ and define $\xi_2 = \sigma \xi_1$, for any standard deviation $\sigma \in \bR_+$.
Consequently, we obtain}
\begin{equation*}
    \lambda = \frac{\norm{e}_{\infty}}{\norm{e}_2} = \frac{\norm{W \xi_1}_{\infty}}{\norm{W \xi_1}_2} = \frac{\norm{W \xi_2}_{\infty}}{\norm{W \xi_2}_2}.
\end{equation*}
This observation underscores the importance of our choice of $\lambda$, as it is independent of $\sigma$ and remains consistent across various noise levels.
\end{remark}

On a different note, we can deduce that the correct support can be fully recovered if the signal-to-noise ratio is sufficiently high, as formally stated in~\Cref{Thm:supp}. Furthermore, $\lambda$ can be conveniently bounded in all cases by
\begin{equation*}
    \frac{1}{\sqrt{n}} \leq \lambda = \frac{\norm{e}_{\infty}}{\norm{e}_2} \leq 1 .
\end{equation*}
A bound for $\lambda$ can be used to guarantee that no false positive support error occurs. Improved bounds can be achieved with additional assumptions on the noise. In fact, we investigate this aspect for the Gaussian case in \Cref{subsec:gaussian}.

\subsection{Estimation error}
We now proceed to show that the estimation error, $\|\hat z - z^*\|_2$, can also be bounded with the same choice of parameter $\lambda$, under the same Assumptions \ref{ass:bound_noise} and \ref{ass:eta}.

\begin{theorem}\label{Thm:sqrt}
Let \Cref{ass:bound_noise} be satisfied and let $\eta$ satisfy \Cref{ass:eta}. Then, for $\hat z$, the solution to \eqref{eq:our}, we get
\begin{equation*}
    \|\hat z - z^*\|_2 \leq (2+ \eta) \epsilon.
\end{equation*}
\end{theorem}

\begin{proof}[Proof of \Cref{Thm:sqrt}]
From the optimality of the solution:
\begin{equation*}
    \|\hat z - W y\|_2 +\lambda\|\hat z\|_1 \leq \|z^* - W y\|_2 + \lambda\|z^*\|_1.
\end{equation*}
Using $W y = z^* + e$, and applying the reverse triangle inequality, we get
\begin{align*}
    \|\hat z - z^* - e\|_2 &\leq \|e\|_2 +\lambda\|z^*\|_1 - \lambda\|\hat z\|_1 \\
    \|\hat z - z^* \|_2 - \|e\|_2 &\leq \|e\|_2 +\lambda\|z^*\|_1 - \lambda\|\hat z\|_1 \\
    \|\hat z - z^* \|_2 &\leq 2\|e\|_2 + \lambda \left( \|z^*\|_1 - \|\hat z\|_1 \right).
\end{align*}
Finally, using Assumptions \ref{ass:bound_noise} and \ref{ass:eta}, we have
\begin{equation*}
    \|\hat z - z^*\| \leq 2\|e\| + \lambda\|z^*\|_1 \leq (2+ \eta) \epsilon,
\end{equation*}
which concludes the proof.
\end{proof}

In Appendix \ref{App:bound}, we present an empirical validation of \Cref{Thm:sqrt}.

\subsection{Gaussian noise} \label{subsec:gaussian}
We now extend the results of \Cref{subsec:support} to the Gaussian case. The key point of this analysis is that we can use practical values for $\lambda$, which can be computed independently of the noise level. We show that these $\lambda$ values are suitable for any additive Gaussian noise in the input and are thus pivotal to its standard deviation $\sigma$.

\begin{assumption}\label{ass:gauss_noise}
The entries of $\xi$ are \iid $\cN(0,\sigma^2)$ random variables.
\end{assumption}

\begin{assumption}\label{ass:norm_w}
The rows of $W$ are normalized to unit $\ell_2$ norm.
\end{assumption}

\begin{theorem} \label{Thm:supp_gauss}
Let \Cref{ass:gauss_noise} and \ref{ass:norm_w} be satisfied, let $\eta$ satisfy \Cref{ass:eta}. Then, set
\begin{equation*}
    \lambda = a \sqrt{ \frac{2 \log d}{n}} ,
\end{equation*}
where $a \geq 2\sqrt{2}$ is a constant.

With probability at least $1 - 2 d^{1- a^2/8} - (1+e^2)e^{-n/24}$, we have
\begin{equation*}
    \norm{\hat z - z^*}_{\infty} \leq \lambda (2 + 2 \eta + \tfrac{1}{s_{\min}}) \sqrt{n} s_{\max} \sigma ,
\end{equation*}
where $s_{\min}$ and $s_{\max}$ are the minimum and maximum
singular values of $W$, respectively.
\end{theorem}

\begin{proof}[Proof of \Cref{Thm:supp_gauss}]
Let $\cA$ be the event
\begin{equation*}
    \cA = \left\{ \tfrac{\norm{e}_{\infty}}{\norm{e}_2} \leq \tfrac{\lambda}{2 s_{\min}} \right\} \cap \left\{ s_{\min} \tfrac{\sigma}{\sqrt{2}} < \tfrac{\norm{e}_2}{\sqrt{n}} < 2 s_{\max} \sigma \right\}.
\end{equation*}
\rev{From \Cref{lem:concentration} presented in Appendix~\ref{App:lemmas}, it is deduced that $\bP(\cA) \geq 1 - 2 d^{1- a^2/8} - (1+e^2)e^{-n/24}$.} For our model and under the event $\cA$, we have
\begin{align} \label{eq:gaus_bound_delta}
    \norm{\hat z - z^*}_\infty \nonumber
    & \leq
    \norm{\hat e}_{\infty} + \norm{e}_{\infty} \nonumber \\
    & \leq
    \norm{\hat e}_{\infty}
    + \lambda \frac{s_{\max}}{s_{\min}} \sqrt{n} \sigma.
\end{align}
By minimality of the estimator and \Cref{ass:eta},
\begin{align} \label{eq:gau_f_sigma}
    \norm{\hat e}_2
    &\leq
    \norm{e}_2
    + \lambda \norm{z^*}_{1} \nonumber \\
    &\leq
    2 \sqrt{n} s_{\max} \sigma + 2  \eta \sqrt{n} s_{\max} \sigma .
\end{align}
Combining equations \eqref{eq:kkt_sqrt}, \eqref{eq:gaus_bound_delta} and \eqref{eq:gau_f_sigma}, we get:
\begin{gather*}
    \norm{\hat z - z^*}_\infty \leq
    \lambda \left( 2 \sqrt{n} s_{\max} \sigma + 2  \eta \sqrt{n} s_{\max} \sigma \right) + \lambda \frac{s_{\max}}{s_{\min}} \sqrt{n} \sigma \\
    \norm{\hat z - z^*}_{\infty} \leq \lambda (2 + 2 \eta + \tfrac{1}{s_{\min}}) \sqrt{n} s_{\max} \sigma .
\end{gather*}
\end{proof}

%% file: 3_algo.tex
\section{Computational algorithm} \label{sec:Algo}
\subsection{An iterative solver}
In this section, we introduce an iterative optimization algorithm for minimizing \eqref{eq:our} that can be efficiently implemented and formulated as a novel sparse auto-encoder architecture. It is worth noting that this objective function corresponds to a convex optimization problem. Therefore, it inherits not only all the theoretical properties of convex optimization problems, but also the algorithms that can be used to solve it, such as the interior point method \cite{Wright2000IPM} or the alternating direction method of multipliers \cite{Boyd2011ADMM}.

In \cite{li2020fast}, the authors studied the geometric structure of the square root lasso problem and concluded that the $\ell_2$ loss function is non-differentiable only in extreme cases of overfitting. In practice, this situation is rare when the data are corrupted by noise and a sufficiently large regularization parameter $\lambda$ is used to produce a sparse solution. Consequently, the data fitting term in the objective function behaves as a strongly convex and smooth function.

\begin{algorithm}[htb]
\caption{Self Normalizing ReLU (NeLU)} \label{Alg:PGD}
    \textbf{Input:}
	\begin{algorithmic}
        \STATE $\bar y \leftarrow W y$, where $W$ is the transform and $y$ its input signal.
	\end{algorithmic}
    
    \textbf{Parameters:}
	\begin{algorithmic}
		\STATE $\lambda$ -- bias.
		\STATE $\beta$ -- step size.
	\end{algorithmic}
	
    \textbf{Output:}
	\begin{algorithmic}
 		\STATE The estimated representation $\hat z$.
	\end{algorithmic}

 	\textbf{Process:}
	\begin{algorithmic}
	    \STATE $\hat z \leftarrow 0$
		\STATE {While not converged:}
		\STATE  \qquad $\hat z \leftarrow S_{\beta\lambda} \left\{ \hat z - \beta\frac{\hat z - \bar y}{\|\hat z - \bar y\|_2} \right\}$
		\STATE \textbf{return} $\hat z$
	\end{algorithmic}
\end{algorithm}

Leveraging these attractive geometric properties, we can use proximal gradient descent to iteratively minimize \eqref{eq:our}. The theoretical analysis in \cite{li2020fast} shows that such an optimization algorithm achieves fast local linear convergence. The same theoretical justification also applies to \eqref{eq:our} since our optimization problem can be viewed as a simpler instance of the square root lasso, presented in \Cref{subsec:synthesis}. Therefore, we propose adapting proximal gradient descent to the problem studied here, as described in Algorithm \ref{Alg:PGD}, which we have named Self Normalizing ReLU or NeLU for short.

\rev{
The algorithm works by continuously refining the solution via iterative updates, progressing in the direction opposite to the gradient of the objective function. Each iterative update incorporates a proximal operator, which introduces a penalty term to the objective function, thereby promoting sparsity. In the specific problem at hand, the proximal operator takes the form of a soft-thresholding operator, as outlined in \Cref{sec:Introduction}. This operator performs a shrinkage operation on the variables, setting any values below a specified absolute threshold to zero. Importantly, the soft-thresholding operator can be replaced with ReLU to enforce nonnegative representations. The iterative process continues until the desired level of convergence is achieved.
}

\subsection{Transform learning}

We propose to adapt \Cref{Alg:PGD} for transform learning by unrolling the algorithm into a layered neural network architecture, following the approach presented in~\cite{Lista2010LeCun}. The idea is to unfold an iterative algorithm and construct it as a network architecture, mapping each iteration to a single operation and stacking a finite number of operations on top of each other. This approach enables the incorporation of a wide range of mathematical techniques into deep learning models~\cite{Yonina2020Unroll, 2020unfolding, Liu2021SGDnet}. Specifically, we achieve this unfolding by limiting the number of iterations in \Cref{Alg:PGD} to $N$ iterations. The resulting architecture is depicted in Figure \ref{fig:nelu}.

\begin{figure}[h]
\centering
\includegraphics[width=\linewidth]{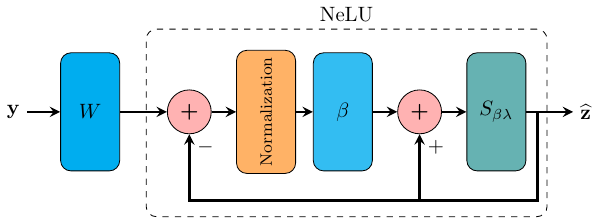}
\caption{The NeLU architecture: A recurrent sparse encoder model, unrolled for a predetermined number of iterations.}
\label{fig:nelu}
\end{figure}

In addition, we propose to improve the performance of the algorithm by employing Nesterov acceleration \cite{sutskever13momentum}. Nesterov acceleration is a variant of momentum that speeds up the convergence of gradient descent algorithms, and has demonstrated efficacy in various contexts. By incorporating it into \Cref{Alg:PGD}, we aim to achieve superior performance in transform learning tasks, based on the understanding that accelerated gradient descent converges faster and operates effectively with shallower networks, which are easier to train.

\begin{algorithm}[t]
\caption{Accelerated NeLU} \label{Alg:Nelu}
    \textbf{Input:}
	\begin{algorithmic}
        \STATE $\bar y \leftarrow W y$, where $W$ is the transform and $y$ its input signal.
	\end{algorithmic}

    \textbf{Learnable weights:}
	\begin{algorithmic}
		\STATE $\lambda$ -- bias.
		\STATE $\beta$ -- step size.
		\STATE $\alpha$ -- momentum factor.
	\end{algorithmic}

    \textbf{Hyperparameters:}
	\begin{algorithmic}
		\STATE $N$ -- number of iterations.
	\end{algorithmic}

    \textbf{Output:}
	\begin{algorithmic}
 		\STATE The estimated representation $\hat z$.
	\end{algorithmic}

	\textbf{Process:}
	\begin{algorithmic}
	    \STATE $\hat z \leftarrow 0$
	 	\STATE $v \leftarrow 0$   
		\STATE {For $i=1:N$}
		\STATE \qquad $v \leftarrow \alpha v - \beta \frac{\hat z + \alpha v - \bar y}{\|\hat z + \alpha v - \bar y\|_2}$
		\STATE \qquad $\hat z \leftarrow \rev{\ReLU \left(\hat z + v - \beta\lambda \right)}$
		\STATE \textbf{return} $\hat z$
	\end{algorithmic}
\end{algorithm}

Finally, we note that all parameters of the resulting algorithm, including $W$, $\lambda$, $\beta$, and $\alpha$, can be trained end-to-end. This means that the network can be trained on a dataset to learn the optimal values for these parameters, allowing it to perform well on various transform learning tasks. The final accelerated algorithm is presented in \Cref{Alg:Nelu}.

\begin{figure*}[hb]
\captionsetup[subfigure]{labelfont=scriptsize}
\centering
\subfloat[]{\includegraphics[width=0.4\linewidth]{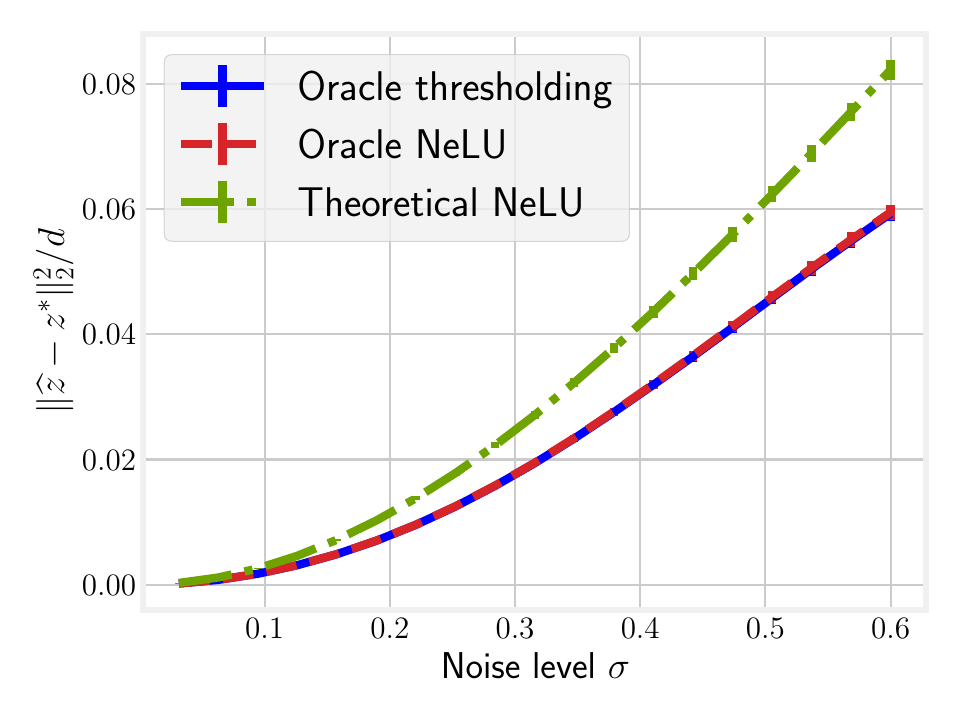}
\label{fig:exp1_n}}
\hfil
\subfloat[]{\includegraphics[width=0.4\linewidth]{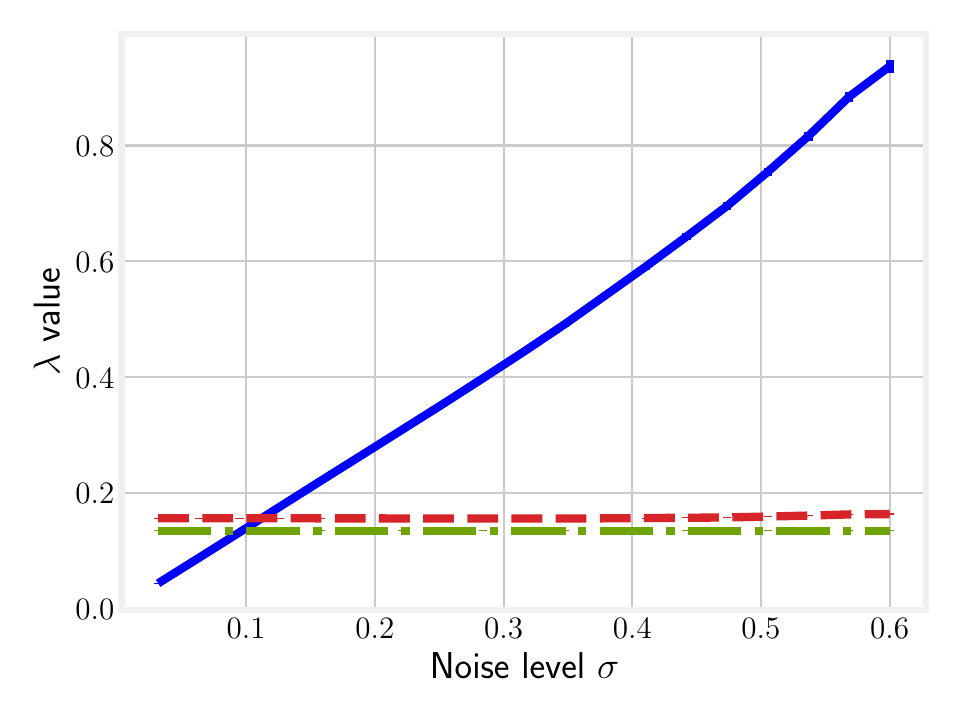}
\label{fig:lam}}
\caption{Experimental results for analytical transform with synthetic data.
(a) Mean squared error (MSE) of $\ell_2$ estimation error as a function of the noise level $\sigma$, evaluated in both settings. In the oracle setting, the regularization parameter is tuned to achieve the smallest estimation error. In the theoretical setting, we use $\lambda = \frac{1}{2}\frac{\norm{e}_{\infty}}{\norm{e}_2}$ in Algorithm~\ref{Alg:PGD}.
(b) $\lambda$ values used for each algorithm in the previous graph. Note that the optimal $\lambda$ values for \Cref{Alg:PGD} are constant, while they are linear for the traditional algorithm. The standard errors are below 0.02 and thus barely visible.}
\label{fig:exp1}
\end{figure*}

At each iteration, the algorithm computes the gradient of the objective function with respect to the model parameters at a point in the direction of the momentum and updates the momentum in the opposite direction of the gradient. \rev{The solution is then updated based to the momentum, taking into account the proximal operator. This operator introduces regularization to prevent overfitting and improve the generalization performance of the model.}

%% file: 4_experiments.tex
\section{Experiments} \label{sec:Exp}

\rev{
In this section, we present experimental results to evaluate the effectiveness of our proposed method under three different settings. First, in \Cref{exp:synthetic}, we use synthetic data to compare the performance of the soft-thresholding algorithm \eqref{eq:encoder} with that of our proposed algorithm \eqref{eq:our}, which we minimize using the iterative approach detailed in \Cref{Alg:PGD}, \emph{given a known transformation}. Next, in \Cref{exp:trainable}, we use the same synthetic data to assess the trainable version of our method, where \emph{the transformation matrix is also learned}, as summarized in \Cref{Alg:Nelu}. Here, we compare our method with a baseline model based on a standard sparse auto-encoder. Lastly, in \Cref{exp:natural}, we evaluate the performance of our trainable \Cref{Alg:Nelu} against a baseline convolutional neural network in the task of image denoising.
}

\subsection{Synthetic data} \label{exp:synthetic}

First, we present experiments conducted on synthetic data to demonstrate the advantage of our proposed method over the traditional sparse encoder algorithm \eqref{eq:encoder}. We assume that the transformation $W$ is known and construct a $100 \times 100$ random matrix, where each entry is sampled from the standard normal distribution, and then normalize the rows to have unit $\ell_2$ norm to satisfy Assumption \ref{ass:norm_w}. Next, we generate the input signal by creating a vector $z^*$ with fixed sparsity level, following the procedure described in \cite{Sulam2020MLPursuit}, to obtain a signal consistent with the transform model, such that $\|Wx\|_0 = 5$. Finally, we contaminate the signal with \iid Gaussian noise of level $\sigma$ to produce the measurements $y = x + \xi$.

We evaluate the estimation error, $\|\hat{z} - z^*\|_2$, which measures the distance between the estimated signal $\hat{z}$ and the true signal $z^*$, as a function of the noise standard deviation $\sigma$. We compare the solutions obtained by minimizing \eqref{eq:encoder} and \eqref{eq:our} in two settings. In the first setting, we perform an oracle cross-validation that sweeps over a range of parameters $\lambda$ to find the regularization parameter that minimizes the estimation error for each algorithm. Importantly, this setting is infeasible since it requires the ground truth data to calculate the estimation error. Nevertheless, it reveals the best performance that one can hope to achieve. In the second setting, we consider a more realistic scenario where we compare the performance of the proposed \Cref{Alg:PGD} using the theoretical value of $\lambda$ divided by 2, following Belloni's empirical improvement~\cite{Belloni2011sqrt}.

Figure \ref{fig:exp1} shows that both algorithms achieve similar performance in the oracle setting. However, the optimal value of $\lambda$ for the soft-thresholding algorithm \eqref{eq:encoder} is linearly proportional to the noise level $\sigma$, while the optimal value for the proposed algorithm \eqref{eq:our} is pivotal to it. This conclusion is consistent with our theoretical analysis presented in \Cref{sec:Theory}. Additionally, we observe that the theoretical value of $\lambda$ for the proposed algorithm, described in \Cref{Thm:supp}, is very close to the actual optimal value. This indicates that the proposed optimization problem may be a reasonable alternative to the current soft-thresholding algorithm, which forms classic encoder architectures, in practical situations where the noise level is unknown.

\subsection{Trainable transforms} \label{exp:trainable}
Trainable transforms often outperform analytical transformations, such as total variation and wavelets, in most signal processing applications~\cite{SparseTransforms}. This is because the transformation used in signal processing is frequently unknown and must be inferred from the data. This motivates the use of neural networks to simultaneously learn the transformation and the sparse representation of the data.

The first learning task is supervised sparse coding, where the input consists of signals $y$ determined by the model and their corresponding synthetic sparse vectors $z^*$ generated by a sparsifying transform, as described in \Cref{exp:synthetic}. \rev{In this case, the goal of the neural network is to learn the transformation $W$ stored in its weights and accurately identify the corresponding sparse output vector given a set of input-output pairs.} Mathematically, this can be expressed as an end-to-end training scheme, minimizing the cost function
\begin{equation*}
    \min_{W, \lambda, \alpha, \beta} \|\hat z - z^*\|^2_2 ,
\end{equation*}
where $\hat z$ is the output of the network outlined in Figure \ref{fig:nelu}.

\begin{figure}[ht]
\centering
\includegraphics[width=0.8\linewidth]{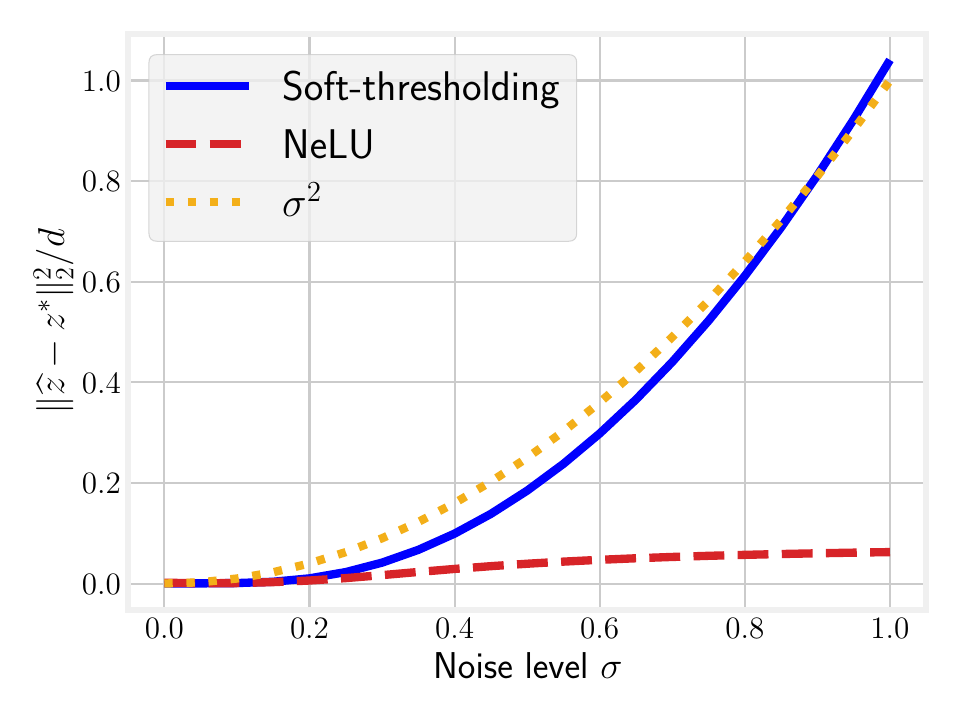}
\caption{Synthetic supervised sparse coding: a comparison of mean squared error (MSE) for estimation error between a two-layer sparse encoder architecture with NeLU and a similar architecture with soft-thresholding, trained on data with a fixed noise level of 0.1. The performance is evaluated at different noise levels, averaged over 2048 realizations of the data.}
\label{fig:rec}
\end{figure}

To compare the effectiveness of the proposed approach, two different neural network architectures are used: one based on Algorithm~\ref{Alg:Nelu}, and a baseline sparse encoder architecture. Both networks consist of a linear layer followed by a thresholding layer. In the baseline version, the thresholding layer is the soft-thresholding non-linearity. In contrast, our proposed architecture uses the Accelerated NeLU non-linearity as presented in \Cref{Alg:Nelu}, \rev{with the soft-thresholding operator}. Both networks are trained using the AdamW optimizer~\cite{2017adamW} and minimize the mean squared error (MSE) loss with a fixed noise standard deviation of $\sigma=0.1$. After training, we evaluate the performance of the fitted networks on different noise levels, $\sigma$, than those used during training. The results, displayed in Figure \ref{fig:rec}, suggest that NeLU is significantly more robust to unseen noise levels than soft-thresholding, indicating that the NeLU non-linearity results in a predictive model that generalizes to other noise levels without additional training.

We also demonstrate the effectiveness of the proposed model in signal denoising applications. In this task, the input is a noisy signal $y = x + \xi$ and the goal is to produce a cleaned version of the signal, $\hat x = W^+ \hat z$, by removing the additive noise. To this end, a final linear layer is added to each network according to the sparse model. The first two layers perform sparse coding, i.e., they estimate the sparse representation, while the last layer projects the sparse estimation back into the input space. Sparse data is generated in the same manner as before and the MSE loss is minimized. The results, shown in Figure \ref{fig:den}, again reveal that the NeLU leads to more stable recovery than the soft-thresholding non-linearity.

\begin{figure}[htb]
\centering
\includegraphics[width=0.8\linewidth]{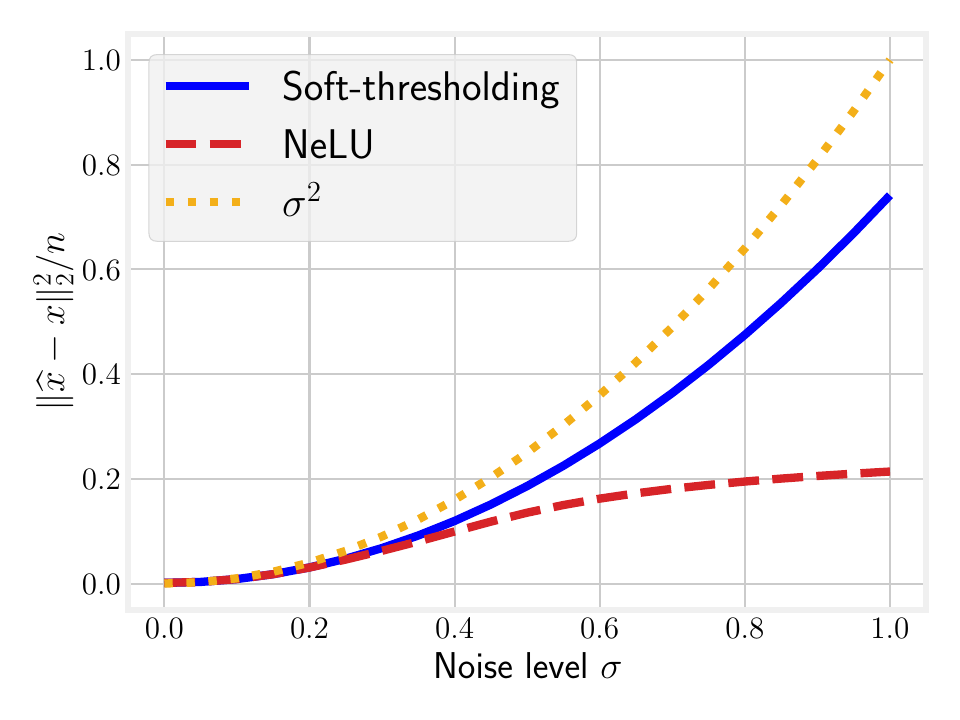}
\caption{Synthetic sparse signal denoising: a comparison of the mean squared error (MSE) for the reconstruction error, $\hat x - x$. Other details are the same as in Figure~\ref{fig:rec}.}
\label{fig:den}
\end{figure}

\subsection{Natural images} \label{exp:natural}
\rev{
In this experiment, the aforementioned networks are employed to perform natural image denoising using a patch averaging technique based on the Convolution Sparse Coding model~\cite{simon2019rethinking}. To accomplish this, each input image is replicated $\textit{stride}$ times and translated across every dimension, producing slight shifts of the original for each replica. As a result, a single image transforms into a set of $\textit{stride}^2$ slightly offset variations. These shifted versions of the input are then processed collectively by the network, resulting in intermediate (shifted) denoised versions of the same input image. Finally, these intermediate denoised output images are shifted back and averaged to yield the final reconstructed output image. This process is visualized in Figure \ref{fig:stride} for a one-dimensional (1D) signal. For more details, see~\cite{simon2019rethinking}.
}

The datasets and preprocessing procedures are adopted from \cite{simon2019rethinking}. Specifically, we use clean training images from the Waterloo Exploration dataset~\cite{Waterloo}, and a validation set consisting of $432$ images from BSD \cite{BSD}. Noisy images are generated by adding white Gaussian noise with a constant standard deviation $\sigma=15$. In each iteration, we randomly crop a patch of size $128^2$ from an image and obtain a random realization of the noise.

In this setting, we replace the linear layers with convolution and deconvolution layers, respectively. Concurrently, the soft-thresholding operator is substituted by ReLU \rev{as the proximal operator}. The models learn 175 filters of dimensions $11\times11$ and a stride of 8. We utilize the AdamW optimizer with a learning rate of $2\cdot10^{-2}$, which is reduced by a factor of 0.7 after every 50 epochs. Additionally, the optimizer's $\epsilon$ parameter is set to $10^{-3}$ to ensure stability. The models are trained for 300 epochs, minimizing the MSE loss.

To evaluate the performance of the models, we use the BSD68 dataset, which is distinct from the validation set. The experimental results, as shown in \Cref{tab:exp3}, allow us to compare the performance of each model on the test dataset at varying noise levels. We can see that the proposed \Cref{Alg:Nelu} layer outperforms the ReLU activation function for virtually all noise levels, and the performance gap widens as the noise level deviates further from the trained noise level.

\begin{table*}[t]
\caption{PSNR comparison of sparse auto-encoder models with NeLU and ReLU activations for natural image denoising. The models are trained on clean-noisy image pairs with a fixed noise level of $\sigma=15$ and evaluated on the test set.}
\label{tab:exp3}
\centering
\small
\renewcommand{\arraystretch}{1.4}
\begin{tabular}{@{}ccccccccc@{}} 
\hline
$\sigma$   & 15   & 25	& 35	& 50	& 75	& 90    & 105 & 120 \\
\hline
Noise & 24.61   & 20.17	& 17.25	& 14.15	& 10.63	& 9.05  & 7.70 & 6.55 \\
ReLU & 28.47 & 25.89	& \textbf{23.49}	& 20.58	& 17.07	& 15.48	& 14.13 & 12.99 \\
NeLU & \textbf{28.65} & \textbf{25.93}	& 23.48	& \textbf{20.69}	& \textbf{17.62}	& \textbf{16.37}	& \textbf{15.36}	& \textbf{14.57} \\
\hline
\end{tabular}
\end{table*}

%% file: 5_conclusion.tex
\section{Conclusion}
In this work, we proposed a novel sparse auto-encoder architecture as an alternative to traditional auto-encoder architectures. We offer a novel activation function, called Self Normalizing ReLU (NeLU), which is the solution of a square root lasso problem under a transform learning formulation. Importantly, as we showed in \Cref{sec:Theory}, the bias parameter of our proposed NeLU layer is pivotal (i.e., invariant) to the noise level in the input signal. This feature leads to an activation function that is significantly more robust to varying noise levels in terms of signal recovery and denoising, both on synthetic data as well as in real imaging settings. Our research showcases how theoretical understanding of neural networks can give rise to improved algorithms, derived from theoretical insights and analysis.

\rev{
Several open questions present opportunities for future directions. While our paper focuses on establishing foundational theory, future efforts might apply these insights on a larger scale to develop a state-of-the-art network. A potential direction would be to extrapolate the model to a multilayer architecture, building upon the work reported in~\cite{Papyan2017convolutional}. The multilayer expansion strategy proposed by \cite{Sulam2018Pursuit} also offers an attractive option. Incorporating additional layers into the model, and possibly broadening the analysis to convolutional neural network (CNN) architectures, could result in improved theoretical bounds and performance.
}

%% file: appendix.tex
\section{Lemmas} \label{App:lemmas}

\begin{lemma} \label{lem:concentration}
Consider a Gaussian vector $\xi \sim \mathcal{N}(0,\sigma^2 I)$ and a deterministic matrix $W$ with normalized rows, where $s_{\min}$ and $s_{\max}$ denote the minimum and maximum singular values of $W$, respectively. Then,
\begin{enumerate}
\item \label{lem:event_lasso}
    Let $\cC_1 \eqdef \left \{\frac{1}{\sqrt{n}} \norm{W \xi}_\infty \leq \frac{\lambda}{2} \right \}$.
    Take $\lambda = a \sigma \sqrt{\frac{\log d}{n}}$ and $a > 2 \sqrt{2}$, then:
    \begin{equation*}
        \bP (\cC_1) \geq 1 - 2 d^{1- a^2/8} .
    \end{equation*}
\item \label{lem:event_sqrt_lasso}
    Let $\cC_2 \eqdef \left \{ \frac{1}{\sqrt{n}} \norm{W \xi}_\infty \leq \frac{\lambda}{2} \frac{\sigma}{\sqrt{2}} \right \}$.
    Take $\lambda = a \sqrt{\frac{2 \log d}{n}}$ and $a > 2 \sqrt{2}$, then:
    \begin{equation*}
        \bP (\cC_2) \geq 1 - 2 d^{1- a^2/8} .
    \end{equation*}
\item \label{lem:event_chi_square}
    Let $\cC_3 \eqdef \left\{ s_{\min} \frac{\sigma}{\sqrt{2}} < \frac{\norm{W \xi}_2}{\sqrt{n}} < 2 s_{\max} \sigma \right\}$, then:
    \begin{equation*}
        \bP (\cC_3) \geq 1 - (1+e^2)e^{-n/24} .
    \end{equation*}
\item \label{lem:control_A}
    Let
    \begin{multline*}
    \cA = \left\{ \frac{\norm{W \xi}_{\infty}}{\norm{W \xi}_2} \leq \frac{\lambda}{2 s_{\min}} \right\} \cap \\
    \left\{ s_{\min} \frac{\sigma}{\sqrt{2}} < \frac{\norm{W \xi}_2}{\sqrt{n}} < 2 s_{\max} \sigma \right\},
    \end{multline*}
    then:
    \begin{equation*}
        \bP(\cA) \geq 1 - \bP (\cC^c_2) - \bP (\cC^c_3) .
    \end{equation*}
\end{enumerate}
\end{lemma}

\begin{proof}
\Cref{lem:event_lasso}:
Since $\xi$ is isotropic, the law of $d^\mathsf{T} \xi$ is the same for all vectors $d \in \bR^n$ of the same norm.
In particular, $W_i \xi$, where $W_i$ is the $i$th row of $W$, and $\xi_{1}$ have the same law.
\begin{align*}
    \bP\left(\cC_1^c\right)
    & \leq \sum^d_i \bP \left(\left|W_{i} \xi \right| \geq \sqrt{n} \lambda / 2 \right) \\
    & \leq d \, \bP \left(|\xi_1 | \geq \sqrt{n} \lambda / 2 \right) & (\norm{W_{i}}_2=1) \\
    & \leq 2 d \exp \left(-\frac{n \lambda^{2}}{8 \sigma^{2}}\right) & \text{(Hoeffding’s Inequality)} \\
    & \leq 2 d^{1-\frac{A^{2}}{8}} & (\lambda = A \sigma \sqrt{\frac{\log d}{n}}) .
\end{align*}
\Cref{lem:event_sqrt_lasso} is a direct consequence of \Cref{lem:event_lasso}.

\Cref{lem:event_chi_square}: 
Giraud~\cite{giraud_statistics} controls the event
\begin{equation*}   
    \left\{ \frac{\sigma}{\sqrt{2}} \leq \frac{\norm{\xi}_2}{\sqrt{n}} \leq \left(2 - \frac{1}{\sqrt{2}}\right) \sigma \right\}.
\end{equation*}
with probability $(1+e^2)e^{-n/24}$. Therefore, we can control our event by
\begin{equation*}
    s_{\min} \frac{\norm{\xi}_2}{\sqrt{n}} \leq \frac{\norm{W \xi}_2}{\sqrt{n}} \leq s_{\max} \frac{\norm{\xi}_2}{\sqrt{n}} .
\end{equation*}

Proof of \Cref{lem:control_A} is done using~\cref{lem:event_sqrt_lasso,lem:event_chi_square}. Indeed we have
\begin{multline*}
    \cA \supset \left\{ s_{\min} \frac{\sigma}{\sqrt{2}} < \frac{\norm{W \xi}_2}{\sqrt{n}} < 2 s_{\max} \sigma \right\} \cap \\
    \left\{ \frac{\norm{W \xi}_{\infty}}{\sqrt{n}} \leq \frac{\lambda \sigma}{2\sqrt{2}} \right\}
     = \cC_2 \cap \cC_3.
\end{multline*}
Hence $\bP(\cA) \geq 1 - \bP(\cC^c_2) - \bP(\cC^c_3)$.
\end{proof}

\begin{remark}
The bounds and probabilities presented in this analysis can be further refined through a more rigorous examination. Nevertheless, these bounds are sufficient for our primary objective, which is to demonstrate the pivotalness of $\lambda$ in the Gaussian scenario.
\end{remark}

\section{Extension to synthesis model} \label{App:synthesis}
Note that if we define $D = W^\mathsf{T}$,
\begin{equation*}
    \bar z = D^\mathsf{T} D z^*.
\end{equation*}
Then the problem \eqref{eq:our} can be rewritten as
\begin{equation*}
    \min_ z \norm{z - \bar y}_2 +\lambda \| z\|_1,
\end{equation*}
where $\bar y = D^\mathsf{T} y = \bar z + e$. This is identical to the square root lasso when the dictionary is the identity matrix. Therefore, it can be solved for $\bar z$ using all the tools available for the square root lasso, as previously studied in \cite{Belloni2011sqrt, massias2020support}. We extend the previous results of the paper to obtain bounds for the synthesis model.

\begin{theorem} \label{Thm:syn_supp}
Following the assumptions defined in \Cref{Thm:supp}, we have
\begin{equation*}
    \|{\hat z -  z^*}\|_\infty \leq \lambda (2 + \eta) \epsilon + \rho \, \mu(D) \| z^*\|_\infty ,
\end{equation*}
where $\rho = \|z^*\|_0$ and $\mu(D)=\max_{i\ne j} |D_{i}^\mathsf{T} D_{j}|$. Moreover, if
\begin{equation} \label{eq:min_signal_estimation}
    \min_{j \in [d]} |z_j^*| > 2 \lambda (2 + \eta) \epsilon + 2\rho \, \mu(D) \| z^*\|_\infty ,
\end{equation}
then the estimated support
\begin{equation*} 
    \hat \cS = \{j \in [d] : |\hat z_j| > \lambda (2 + \eta) \epsilon + \rho \, \mu(D) \| z^*\|_\infty \}
\end{equation*}
recovers the true sparsity pattern correctly, i.e., $\hat \cS = \cS$.
\end{theorem}

\begin{proof}[Proof of \Cref{Thm:syn_supp}]
Using \Cref{Thm:supp} and the sparsity of $z^*$, we have
\begin{align*}
  \norm{\hat z -  z^*}_{\infty}
  &\leq
  \norm{\hat z - \bar z}_{\infty} + \norm{\bar z -  z^*}_{\infty} \\
  &\leq \lambda (2 + \eta) \epsilon + \norm{(I - D^\mathsf{T} D)  z^*}_{\infty} \\ 
  &= \lambda (2 + \eta) \epsilon + \max_i |(I - D^\mathsf{T} D)_\cS  z^*_\cS|_i \\
  &\leq \lambda (2 + \eta) \epsilon + \max_i \|(I - D^\mathsf{T} D)_{\cS i}\|_1 \| z^*\|_\infty \\
  &\leq \lambda (2 + \eta) \epsilon + \rho \, \mu(D) \|z^*\|_\infty .
\end{align*}
This proves the bound on $\norm{\hat z -  z^*}_\infty$.
Then, the support recovery property easily follows as in \Cref{Thm:supp}.
\end{proof}

\begin{corollary}
From the condition \eqref{eq:min_signal_estimation}, we can derive a bound on the sparsity of the signal for correct support recovery:
\begin{gather*}
    |z^*_{\min}| > 2 \lambda (2 + \eta) \epsilon + 2 \rho \, \mu(D) | z^*_{\max}|, \\
    \rho < \frac{| z^*_{\min}|}{2\mu(D) | z^*_{\max}|} - \frac{\lambda(2 + \eta)\epsilon}{\mu(D)| z^*_{\max}|},
\end{gather*}
where $|z^*_{\min}|$ and $|z^*_{\max}|$ are the minimum and maximum absolute values of the entries of $ z^*$, respectively. This bound closely resembles the optimality condition of the thresholding lasso algorithm \cite{Papyan2017convolutional}.
\end{corollary}

\begin{theorem}\label{Thm:syn_sqrt}
Following the assumptions defined in \Cref{Thm:sqrt}, then, we get
\begin{equation*}
    \|{\hat z -  z^*}\|_2 \leq (2 + \eta) \epsilon + \mu(D) \sqrt{d} \| z^*\|_1 .
\end{equation*}
\end{theorem}

\begin{proof}[Proof of \Cref{Thm:syn_sqrt}]
Using \Cref{Thm:sqrt} and the sparsity of $ z^*$, we have:
\begin{align*}
  \norm{\hat z -  z^*}_2
  &\leq
  \norm{\hat z - \bar z}_2 + \norm{\bar z -  z^*}_2 \\
  &\leq
  (2 + \eta) \epsilon + \norm{(I - D^\mathsf{T} D)  z^*}_2 \\
  &\leq
  (2 + \eta) \epsilon + \sqrt{\sum^d_{i=1}{\mu(D)^2 \| z^*\|_1^2}} \\
  &\leq
  (2 + \eta) \epsilon + \mu(D) \sqrt{d} \| z^*\|_1 .
\end{align*}
\end{proof}

\section{Illustration of experiment in section \ref{exp:natural}}
\rev{
In the real-data experiment described in \Cref{exp:natural}, we follow the approach of Simon and Elad for deploying the Convolutional Sparse Coding (CSC) model [1]. The process of this experiment is illustrated in \Cref{fig:stride}.
}

\begin{figure*}[htb]
\centering
\includegraphics[width=.8\linewidth]{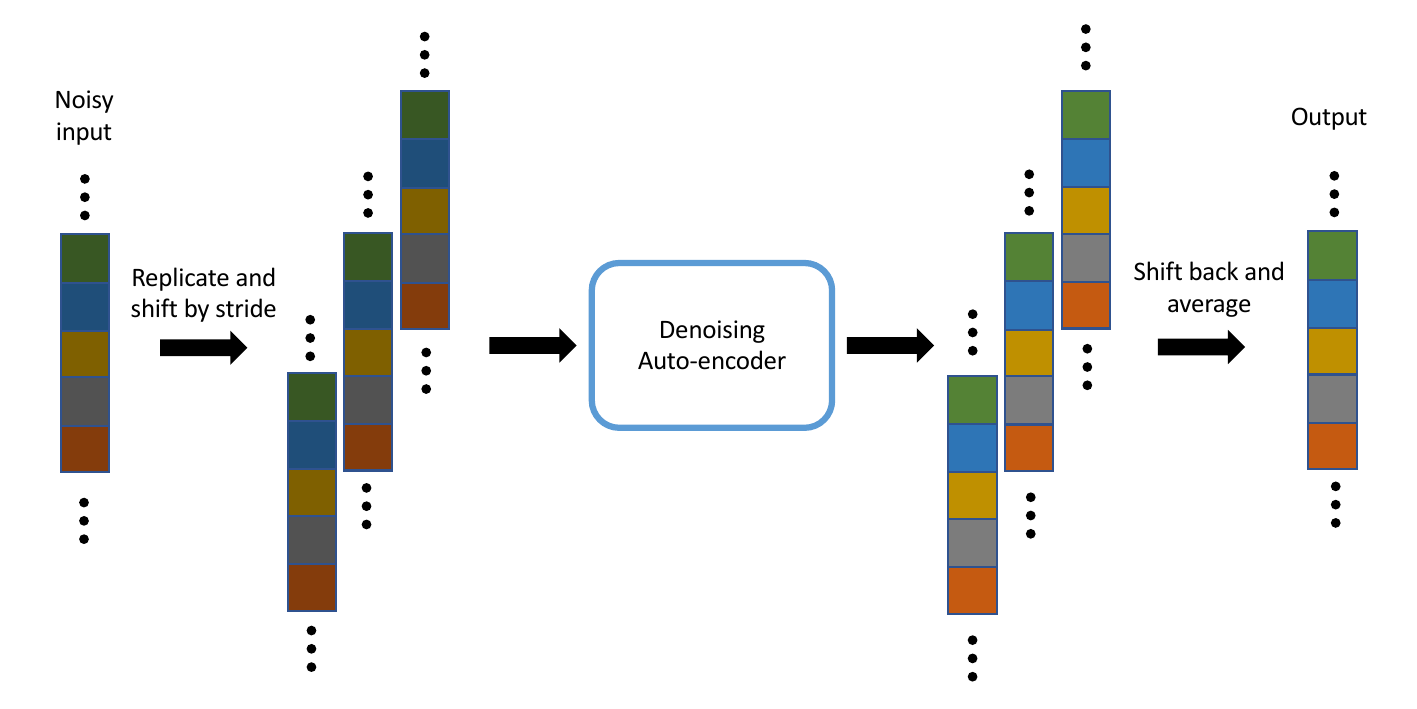}
\caption{1D example of the process utilized in the experiment in \Cref{exp:natural} when $\textit{stride}=2$. Each signal is replicated $\textit{stride}$ times and subsequently translated, yielding slight shifts of the original for each replica. This process effectively transforms a single image into a collection of $\textit{stride}$ variations, each exhibiting a slight spatial offset. The final output is an average of all denoised shifts.}
\label{fig:stride}
\end{figure*}

\section{Details of the analytical transform experiment - \texorpdfstring{$\ell_\infty$}{} error}
\rev{
In this section, we delve deeper into the performance of the algorithms from the main experiment, focusing on the $\ell_\infty$ estimation error. This analysis complements our earlier discussion centered around the $\ell_2$ error, as shown in \Cref{fig:exp1}.
}

\begin{figure}[ht]
\centering
\includegraphics[width=0.8\linewidth]{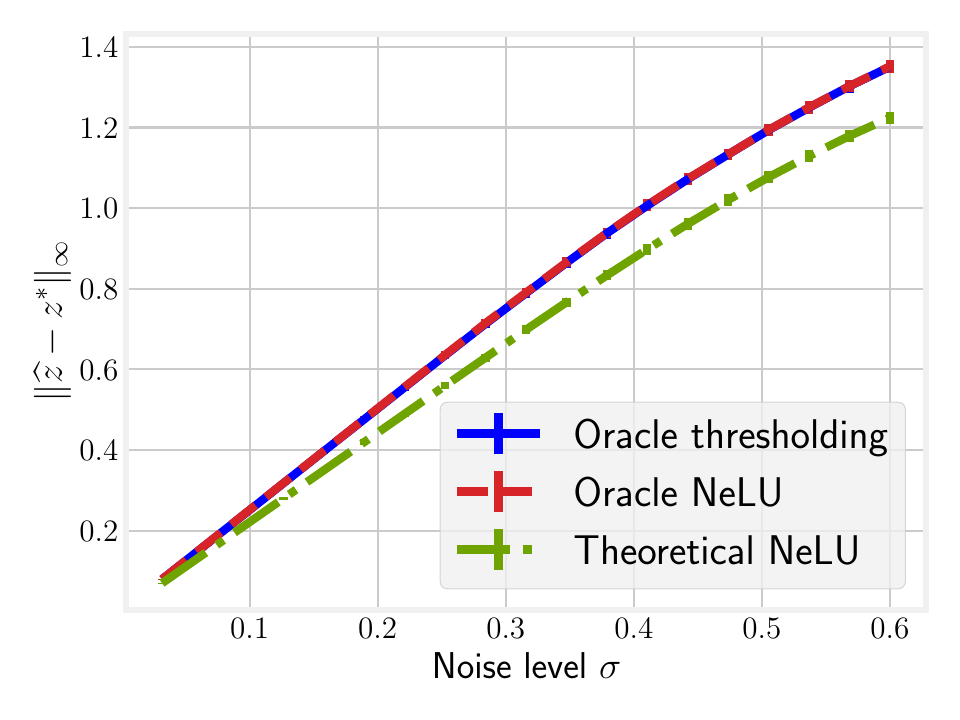}
\caption{The $\ell_\infty$ estimation error of each algorithm in the experiment presented in \Cref{fig:exp1}. The theoretical NeLU achieves a comparatively lower $\ell_\infty$ error because the oracle algorithms were optimized using the $\ell_2$ error. The standard errors are less than 0.02 and are thus barely noticeable.}
\label{fig:exp1_app}
\end{figure}

\rev{
As observed in \Cref{fig:exp1_app}, the theoretical NeLU outperforms other methods, yielding a lower $\ell_\infty$ error. This performance can be attributed to the fact that the oracle algorithms, in the primary experiment, were fine-tuned using the $\ell_2$ error criterion. A parallel behaviour is observed for the $\ell_2$ error when optimizing with respect to the $\ell_\infty$ error.
}

\section{Qualitative analysis}
\rev{
This appendix provides a visual representation of the outcomes from the experiments detailed in Section \ref{exp:natural}. The emphasis here is on a qualitative assessment of images processed using the networks specified in our study. Such an analysis aids comprehending the practical efficacy of the applied denoising methods.
}

\rev{
Figure \ref{fig:qual} showcases a comprehensive comparison, juxtaposing the original images with their noisy counterparts and the subsequent denoised versions. This layout facilitates a direct visual evaluation of the noise reduction capabilities of the networks. For each image displayed, a specific patch has been chosen for detailed analysis. Specifically, for each image, the selection includes the original image marked with the patch location, the unaltered patch, the corresponding noisy patch (the original with added Gaussian noise), and the denoised version processed by each network. Additionally, a heatmap of the residual error is presented, enabling a more precise and detailed comparison.
}

\begin{figure*}[htb]
\centering
\includegraphics[width=.9\linewidth]{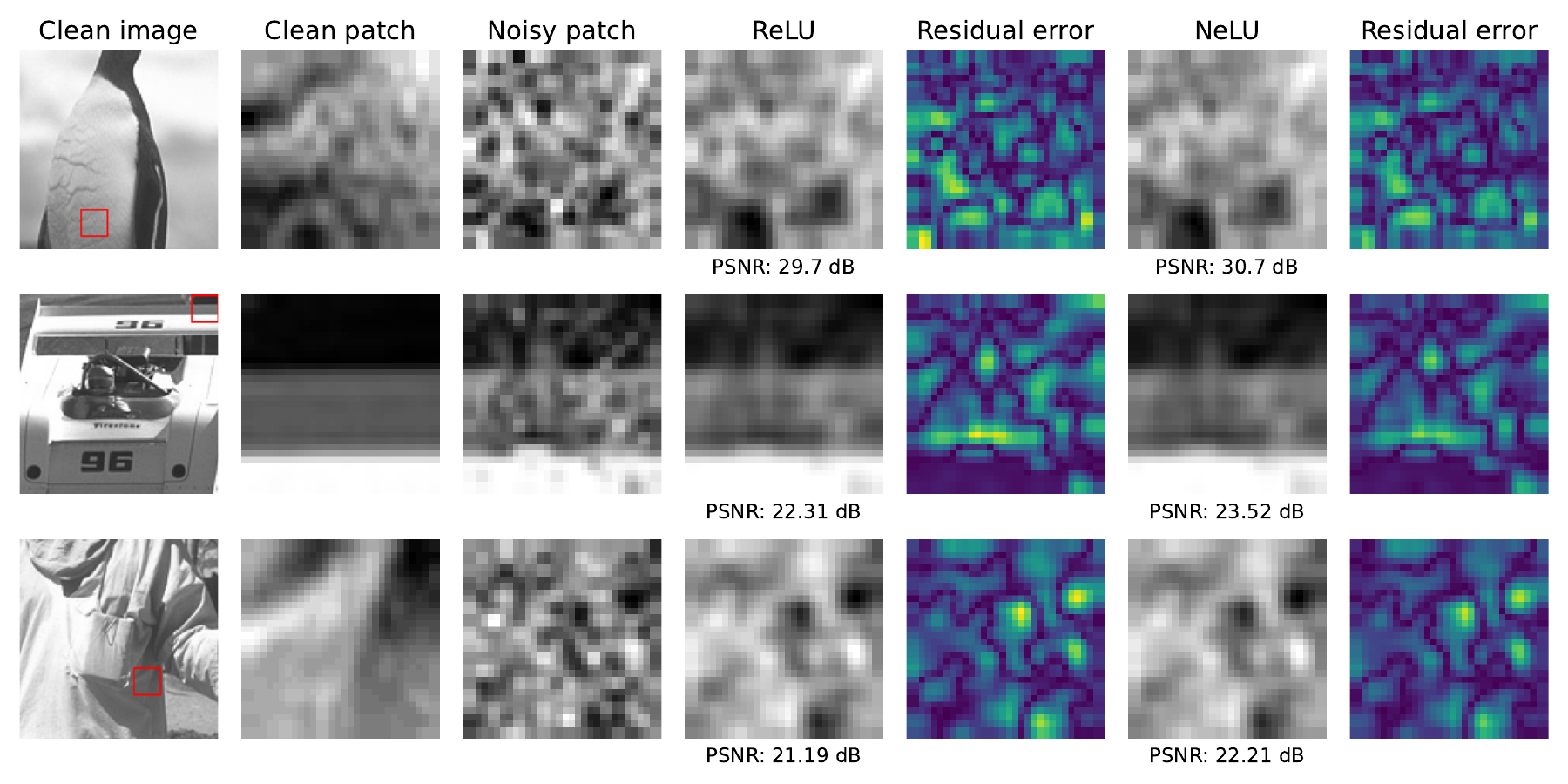}
\caption{Qualitative analysis of denoising networks from \Cref{exp:natural}. For each image, a selected patch is displayed in three states: the original, the version with Gaussian noise ($\sigma=25$ for the first image, $\sigma=50$ for the second and third), and the denoised output from each evaluated network. Additionally, heatmaps depict the residual errors, with warmer colors indicating larger absolute errors in those pixels. The Peak Signal-to-Noise Ratio (PSNR) values are also reported below each network's output to quantitatively assess the denoising performance.}
\label{fig:qual}
\end{figure*}

\rev{
The residual error visualized in the figure highlights the enhancements our algorithm achieved, particularly in handling previously unseen noise levels. This visual representation serves not only as a validation of the algorithm's effectiveness but also offers insights into its potential limitations and areas for future improvement.
}

\section{Empirical validation of Theorem \ref{Thm:sqrt}}
\label{App:bound}
\rev{
In this appendix, we extend our investigation to empirically validate the theoretical estimation error bound introduced within our theoretical framework. To accomplish this, we replicate the experimental setup delineated in \Cref{exp:synthetic}. Herein, we evaluate the performance of our proposed algorithm \eqref{eq:our} in juxtaposition with the theoretical threshold. Employing the iterative process specified in \Cref{Alg:PGD}, and considering a given transformation, we subject the algorithm to rigorous evaluation across a range of noise levels. For each noise level, the experiment encompasses 20 independent trials, incorporating diverse instances of signal and noise, with $\lambda$ set to $\frac{\norm{e}_{\infty}}{\norm{e}_2}$.
}

\begin{figure}[ht]
\centering
\includegraphics[width=0.8\linewidth]{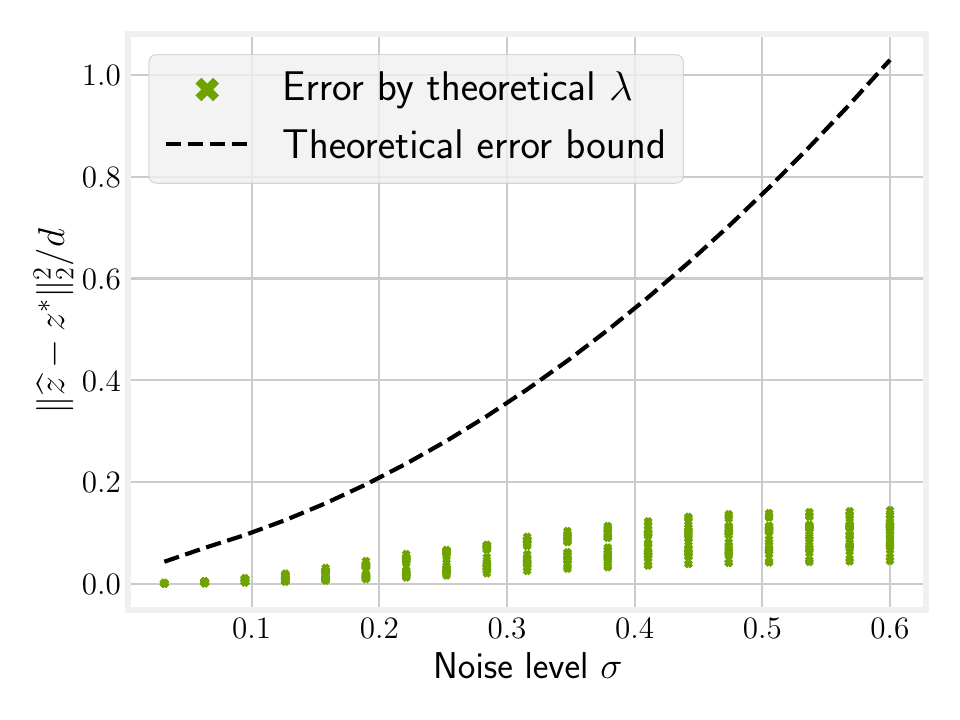}
\caption{Empirical validation of the theoretical estimation error bound. The figure showcases the MSE of the $\ell_2$ estimation error as a function of the noise level $\sigma$. For each noise level, 20 independent trials were conducted with distinct combinations of signal and noise, with $\lambda = \frac{\norm{e}_{\infty}}{\norm{e}_2}$.}
\label{fig:bound}
\end{figure}

\rev{
The outcomes of these experiments are illustrated in \Cref{fig:bound}. These experiments corroborate the assertion that the reconstruction error adheres to the bounds established in \Cref{Thm:sqrt}, thereby reinforcing the theoretical underpinnings of our methodology. Despite being conservative, the framework provides a reliable method for ensuring model robustness across different noise levels.
}

\rev{
Further investigation into the exact optimal value of $\lambda$, considering factors such as the sparsity level of the signals, presents a promising avenue for future research. It is our hope that these insights will contribute to a deeper understanding of the interplay between theoretical analysis and empirical application in the domain of sparse autoencoders.
}